\documentclass[journal]{IEEEtran}
\usepackage{microtype}
\usepackage{graphicx}
\usepackage{subfigure}
\usepackage{booktabs} 
\usepackage{amsmath}
\usepackage{amsthm}	
\usepackage{amsfonts}

\usepackage[utf8]{inputenc} 
\usepackage[T1]{fontenc}    
\usepackage{hyperref}       
\usepackage{url}            
\usepackage{booktabs}       
\usepackage{amsfonts}       
\usepackage{nicefrac}       
\usepackage{microtype}      

\usepackage{amssymb}

\usepackage{fancyhdr}
\ifCLASSINFOpdf
\else
\fi
\begin{document}
%
\title{Understanding Global Loss Landscape of One-hidden-layer ReLU Networks\\ Part 1: Theory}
%
%
%

\author{Bo Liu
	\thanks{Bo Liu is with College of Computer Science, Faculty of Information Technology, Beijing University of Technology, Beijing, China. e-mail: liubo@bjut.edu.cn.}}

\maketitle

\begin{abstract}
For one-hidden-layer ReLU networks, we prove that all differentiable local minima are global inside differentiable regions. We give the locations and losses of differentiable local minima, and show that these local minima can be isolated points or continuous hyperplanes, depending on an interplay between data, activation pattern of hidden neurons and network size. Furthermore, we give necessary and sufficient conditions for the existence of saddle points as well as non-differentiable local minima, and their locations if they exist.
\end{abstract}


%
\IEEEpeerreviewmaketitle

\section{Introduction}
\label{Introduction}

\IEEEPARstart{O}{ne} of the greatest mysteries in deep learning is the non-convex global loss landscape of deep neural networks. Understanding the global landscape of loss functions, especially whether bad local minima and saddle points exist, their count and locations if they do exist, will not only contribute to understanding the performance of popular local search based optimization methods \cite{RuoyuSunSurvey} such as gradient descent from a geometric point of view, but also can inspire new search algorithms that are guaranteed to escape all bad local minima and saddle points effectively and converge efficiently.

It has been shown that there are no bad local minima for some specific types of networks, including deep linear networks, one-hidden-layer networks with quadratic activations, ultra-wide networks, and networks with special type of extra neurons (see section \ref{section8} for related works). In other words, for these networks all local minima are global, hence there is no chance of getting stuck in bad local minima for local search based optimiztion methods. 

Unfortunately, for ReLU networks that are most widely used in practice, this no bad local minima property does not hold anymore, as evidenced in the studies of e.g., \cite{SafranShamir18,Swirszcz,ChulheeYun19,CRITICALPOINTS,BoundsDescentPathsICLR20}. However, these works either constructed concrete data examples and networks or performed experiments to demonstrate the existence of bad local minima for one-hidden-layer ReLU networks. So far a general theory of existence of bad local minima in ReLU networks was still missing. The weight space of ReLU networks is divided into differentiable regions and non-differentiable boundaries between them due to the non-smoothness introduced by ReLU activation. They were unclear in theory that for ReLU networks of any size and any input data, under what conditions there exist differentiable local minima, and under what conditions there exist non-differentiable local minima and saddle points, and their count and locations if they do exist. It was also unclear that beyond small regions surrounding global minima, how big the probability of existing local minima is at any location in the whole weight space.

In this work, we seek to understand the global loss landscape of one-hidden-layer ReLU networks and answer the above theoretical questions, in the hope of giving inspirations to the understanding of general deep ReLU networks. More specifically, for one-hidden-layer ReLU networks of any size (not just over-parameterized case where network size is bigger than the number of samples) and any input, we have made the following contributions in this paper.

\begin{itemize}
	\item We prove that in differentiable regions all local minima are global (i.e., there are no bad local minima in differentiable regions). We show that local minima can be isolated points or continuous hyperplanes, depending on an interplay between data, activation pattern of hidden-layer neurons and network size. The conditions for existing differentiable local minima and their locations are given.  
	\item We give necessary and sufficient conditions for the existence of saddle points and their locations.
	\item We give necessary and sufficient conditions for the existence of non-differentiable local minima that lie on the boundaries between differentiable regions, and give their locations if they do exist.		
	
\end{itemize}

This paper is organized as follows. Section \ref{section2} describes the one-hidden-layer ReLU network model and gives some preliminaries on Moore-Penrose inverse. In section \ref{section3}, we prove that all local minima are global in differentiable regions. Section \ref{section4} gives the locations of differentiable local minima and presents conditions for the existence of genuine differentiable local minima, and illustrates the single point and continuous cases of local minima with a simple example. We give the necessary and sufficient conditions for saddle points in section \ref{section5}, and for non-differentiable local minima in section \ref{section6}. Missing proofs are given in section \ref{section7}. Section \ref{section8} is related work.

\section{One-hidden-layer ReLU Neural Network Model and Preliminaries} \label{section2}
\subsection{One-hidden-layer ReLU Neural Networks}
\label{section2.1}
In the one-hidden-layer ReLU network model studied in this paper, suppose there are $ K $ hidden neurons with ReLU activations, \textit{d} input neurons and a single output neuron. We use $[N]$ to denote $\left \{ 1,2,\cdots,N \right \}$. The input samples are $ \left \{ (\mathbf{x}_i,\ y_i)\ ,i\in [N] \right \} $, where $ \mathbf{x}_i\in \mathbb{R}^d$ is the \textit{i}th homogeneous data vector (i.e., augmented with scalar 1) and $ y_i\in\pm1 $ is the label of  $ \mathbf{x}_i $. We make no assumptions on the network size and input data. Denoting the weight vectors connecting hidden neurons and input as $ \left \{ \mathbf{w}_i,\ i\in [K] \right \} $ (augmented with bias), and the weights between output neuron and hidden ones as $ \left \{ z_i,\ i\in [K] \right \} $, the loss of one-hidden-layer ReLU networks is
\begin{equation}\label{eq1}
L(z,\mathbf{{w}})=\frac{1}{N}\sum_{i=1}^{N}l(\sum_{j=1}^{K}z_j\cdot \begin{bmatrix}
\mathbf{w}_j\cdot \mathbf{x}_i
\end{bmatrix}_{+},y_i),
\end{equation}
where $z=\left\{z_k,\ k\in [K]\right\}$, $\mathbf{w}= \left\{\mathbf{w}_k,\ k\in [K]\right\}$, $ [y]_+=max(0,y) $ is the ReLU function and \textit{l} is the loss function. We assume \textit{l} is convex, which is true for the commonly used squared loss and cross-entropy loss.

\subsection{Moore-Penrose Inverse}
\label{section2.2}
Moore-Penrose inverse of matrices \cite{MatrixAnalysis} will be heavily used in this paper. $M^+$ denotes the Moore-Penrose inverse of a matrix $ M\in\ \mathbb{R}^{m\times n}$. It satisfies the following four equations: $MM^+M=M$, $M^+MM^+=M^+$, $({MM^+})^T=MM^+$, $({M^+M})^T=M^+M$. Therefore, $M^+=0$  if $M=0$. Moore-Penrose inverse has the following properties that will be useful in this paper: $ A^+=(A^TA)^+A^T $, $ (A^+)^T=(A^T)^+ $, $MM^+=I_m$ if and only if $rank(M)=m$, $M^+M=I_n$ if and only if $rank(M)=n$, where $I_m$ is the $m\times m$ identity matrix. If $M\in \mathbb{R}_{r}^{m\times n}$ ($r> 0$ is the rank of $ M $), and the full-rank decomposition of $ M $ is $M=FG \ (F\in \mathbb{R}_{r}^{m\times r}, G\in \mathbb{R}_{r}^{r\times n})$, then $M^+=G^T{(GG^T)}^{-1}{(F^TF)}^{-1}F^T$. For $\mathbf{b}\in \mathbb{R}^m$, the general solution to the least square problem $\min_{\mathbf{z}} \begin{Vmatrix}
\mathit{M}\mathbf{z}-\mathbf{b}
\end{Vmatrix}_{2}^{2}$ is $\mathbf{z}=M^+\mathbf{b} + \left(I-M^+M\right)\mathbf{c}$, ($\mathbf{c}\in \mathbb{R}^n$ is arbitrary). The necessary and sufficient condition for the linear system $M\mathbf{z}=\mathbf{b}$ to be solvable is $MM^+\mathbf{b}=\mathbf{b}$,  and the general solution is also $\mathbf{z}=M^+\mathbf{b} + \left(I-M^+M\right)\mathbf{c}$.

\section{All Differentiable Local Minima Are Global}
\label{section3}
Let us rewrite the loss into a form that will simplify our problems. Introducing variables $I_{ij}$ which equal 1 if $\mathbf{w}_j \cdot \mathbf{x}_i>0$ and 0 otherwise, the loss can be rewritten as
\[L(z,\mathbf{{w}})=\frac{1}{N}\sum_{i=1}^{N}l(\sum_{j=1}^{K}z_j\cdot I_{ij} \mathbf{w}_j\cdot \mathbf{x}_i,y_i). \]
Defining $\mathbf{R}_j=z_j \mathbf{w}_j$, the loss is converted into
\begin{equation}\label{eq2}L(\mathbf{R})=\frac{1}{N}\sum_{i=1}^{N}l(\sum_{j=1}^{K} I_{ij} \mathbf{R}_j\cdot\mathbf{x}_i,y_i), 
\end{equation}
where $\mathbf{R}= \left\{\mathbf{R}_k,\ k\in [K]\right\}$. This conversion integrates the weights of two layers and is key to our proofs later in this paper.

For one-hidden-layer ReLU network model, sample $\mathbf{x}_i$ is a hyperplane in the space of $\mathbf{w}$, and samples $ \left \{ \mathbf{x}_i,\ i\in [N] \right \} $ partition the $\mathbf{w}$ space into a number of convex cells, such as cell 1 and cell 2 shown in Fig.1. Each weight vector $\mathbf{w}_j$ is therefore located in a certain cell or on the boundary of cells. If all weights $ \left \{ \mathbf{w}_j,\ j\in [K] \right \} $ are located inside cells and move within them without crossing the boundaries, $\left \{ I_{ij},\ i\in [N], j\in [K] \right \} $ will have constant values, and thus loss $ L $ is a differentiable function of $ \left \{ \mathbf{R}_j,\ j\in [K] \right \} $ within these cells. We call the cells $\left \{ \mathbf{w}_j,\ j\in [K] \right \}$ reside in as their \textbf{defining cells}, which can be specified by $\left \{ I_{ij},\ i\in [N], j\in [K] \right \} $. When crossing the boundary of two cells, such as moving $\mathbf{w}_1$ from cell 2 and to cell 1 in Fig.1, $I_{21}$ will change from 1 to 0 at the boundary. Therefore, loss $\mathit{L}$ is non-differentiable on the boundaries. 

Local minima $ (z^\ast, \mathbf{w}^\ast) $ may exist inside cells (each $ \mathbf{w}_j^\ast $ of local minima is inside a certain cell) or on the boundaries (at least one $ \mathbf{w}_j^\ast $ of local minima is on the boundary), and we call them \textbf{differentiable and non-differentiable local minima} respectively. Global landscape of $ L(z, \mathbf{w}) $ consists of local landscapes inside cells (each $ \mathbf{w}_j $ is inside a certain cell) and boundaries between them (at least one $ \mathbf{w}_j $ is on the boundary). In this section, we deal with local landscapes and minima inside cells specified by any feasible $\left \{ I_{ij},\ i\in [N], j\in [K] \right \} $ (by feasible we mean for given samples the values of $\left \{ I_{ij},\ i\in [N], j\in [K] \right \} $ can be achieved by certain $\left \{ \mathbf{w}_j,\ j\in [K] \right \}$). Local minima existing on cell boundaries will be discussed in section \ref{section6}.

\begin{figure}[ht]
	\vskip 0.2in
	\begin{center}
		\centerline{\includegraphics[width=3.0cm]{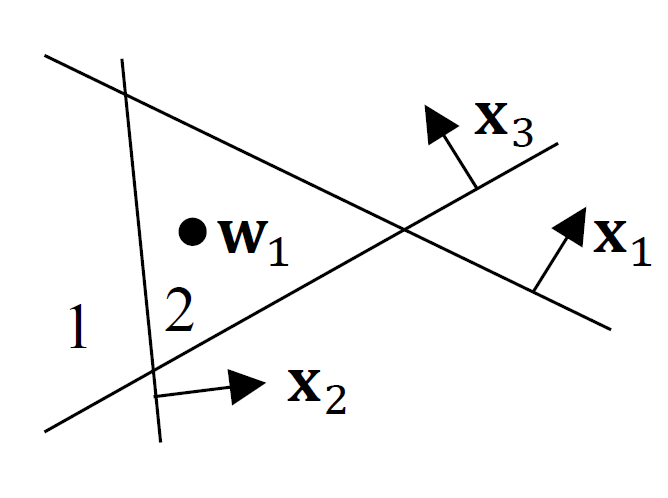}}
		\caption{Samples partition weight space into cells.}
		\label{celldef}
	\end{center}
	\vskip -0.2in
\end{figure}

In this section, we are going to prove that inside cells specified by any feasible constant $\left \{ I_{ij},\ i\in [N], j\in [K] \right \} $, all differentiable local minima are global, namely, there are no bad local minima for the local landscapes inside cells. Notice that despite differentiable minima are global in local landcapes, they might be bad local minima in the global landscape in the sense that their loss might be worse than that of differentiable minima in local landscapes of other cells.  

The core idea is to first prove that differentiable local minima of $ L(z, \mathbf{w}) $ inside cells will lead to $\frac{\partial L}{\partial \mathbf{R}_j}=0\ (j\in [K])$, then by convexity of $ L(\mathbf{R}) $ for constant $\left \{ I_{ij},\ i\in [N], j\in [K] \right \} $ and the fact that stationary point of a convex function corresponds to its unique global minimum, the desired conclusion is followed. 
\newtheorem{lemma}{Lemma}
\begin{lemma}
	Any differentiable local minimum of $L(z,\mathbf{{w}})$ in \eqref{eq1} corresponds to a stationary point of $L(\mathbf{{R}})$ in \eqref{eq2}, that is, $\frac{\partial L}{\partial \mathbf{R}_j}=0\ (j\in [K])$.
\end{lemma}

The following Theorem 1 establishes the globalness of differentiable local minima.

\newtheorem{thm}{Theorem}	
\begin{thm}	
	If loss function \textit{l} is convex, then inside cells specified by any feasible constant $\left \{ I_{ij},\ i\in [N], j\in [K] \right \} $, the differentiable local minimum of $L(z,\textbf{w})$ is global. Furthermore, $L(z,\textbf{w})$ has no differentiable local maxima.
\end{thm}

Despite inside cells $L(\mathbf{R})$ has a unique global minimum, $(z,\textbf{w})$ that achieves global minimal loss is not unique. Due to $\mathbf{R}_j=z_j \textbf{w}_j = cz_j\cdot \frac{1}{c} \textbf{w}_j$ if $c\neq 0$, as a result, if $\left \{ z_j,\textbf{w}_j \right \}$ achieves global minimal loss, so does $\left \{ cz_j,\frac{1}{c} \textbf{w}_j \right \}$. Moreover, although $L(z,\textbf{w})$ has no differentiable local maxima, it may have saddle points, which will be explored in detail in section \ref{section5}.

\subsection{Proofs of Lemma 1 and Theorem 1}\label{section3.1}
\begin{proof}[\indent \bfseries Proof of Lemma 1]
	After introducing variables $\left \{ I_{ij} \right \} $ and defining $\mathbf{R}_j=z_j\cdot\mathbf{w}_j$, the loss of one-hidden-layer ReLU networks has already been given in \eqref{eq2}. Notice that inside cells specified by any feasible constant $\left \{ I_{ij},\ i\in [N], j\in [K] \right \} $, $ L(\mathbf{R}) $ is a differentiable function of only $ \mathbf{R} $. 
	
	At any differentiable local minimum $\hat{z}=\left \{ \hat{z}_k,k\in [K] \right \}$ and $\hat{\mathbf{w}}=\left \{ \hat{\mathbf{w}}_k,k\in [K] \right \}$, the derivatives $\left \{ \frac{\partial L}{\partial z_j},\ \frac{\partial L}{\partial\mathbf{w}_j} \right \}  $ exist and are all equal to 0. By $\mathbf{R}_j=z_j\cdot\mathbf{w}_j$, we have for each $j\in[K]$,  
	\begin{equation}\label{A.3}
	\frac{\partial L}{\partial z_j}\left(\hat{z},\hat{\mathbf{w}}\right)=\frac{\partial L}{\partial\mathbf{R}_j}\left({\hat{\mathbf{R}}}_1,{\hat{\mathbf{R}}}_2,\cdots,{\hat{\mathbf{R}}}_\mathbf{j},\cdots{\hat{\mathbf{R}}}_K\right)\cdot{\hat{\mathbf{w}}}_j=0, 
	\end{equation}
	\begin{equation}\label{A.4}
	\frac{\partial L}{\partial\mathbf{w}_j}\left(\hat{z},\hat{\mathbf{w}}\right)=\frac{\partial L}{\partial\mathbf{R}_j}\left({\hat{\mathbf{R}}}_1,{\hat{\mathbf{R}}}_2,\cdots,{\hat{\mathbf{R}}}_\mathbf{j},\cdots{\hat{\mathbf{R}}}_K\right)\cdot{\hat{z}}_j=0,    
	\end{equation}
	where ${\hat{\mathbf{R}}}_\mathbf{j}={\hat{z}}_j\cdot{\hat{\mathbf{w}}}_j$. If ${\hat{z}}_j\neq0$, \eqref{A.4} implies $\frac{\partial L}{\partial\mathbf{R}_j}=0$, and \eqref{A.3} will be satisfied automatically. If ${\hat{z}}_j=0$, \eqref{A.4} is satisfied, we only need to prove $\frac{\partial L}{\partial\mathbf{R}_j}=0$ from \eqref{A.3} for the case of ${\hat{z}}_j=0$. \\
	
	Since $\left\{\hat{z},\hat{\mathbf{w}}\right\}$ is a local minima of $L$, by definition, there exists $\varepsilon>0$ such that for all $\mathbf{w}_j$ that satisfy $\left \| \mathbf{w}_j-\hat{\mathbf{w}}_j \right \|_2\leqslant \varepsilon$, the following holds
	\begin{equation}\label{A.5}
	\begin{split}
	L({\hat{z}}_1,{\hat{z}}_2,\cdots{\hat{z}}_K,{\hat{\mathbf{w}}}_1,{\hat{\mathbf{w}}}_2,\cdots\mathbf{w}_j,\cdots{\hat{\mathbf{w}}}_K)\geq \\  L({\hat{z}}_1,{\hat{z}}_2,\cdots{\hat{z}}_K,{\hat{\mathbf{w}}}_1,{\hat{\mathbf{w}}}_2,\cdots{\hat{\mathbf{w}}}_j,\cdots{\hat{\mathbf{w}}}_K). 
	\end{split}
	\end{equation}
	We now perturbate ${\hat{\mathbf{w}}}_j$ to $\mathbf{w}_j^\prime={\hat{\mathbf{w}}}_j+\frac{\varepsilon}{2}\mathbf{u}$, and keep $ \left\{ \hat{z}_1,\hat{z}_2,\cdots \hat{z}_K,{\hat{\mathbf{w}}}_1,{\hat{\mathbf{w}}}_2,\cdots{\hat{\mathbf{w}}}_{j-1},{\hat{\mathbf{w}}}_{j+1}\cdots{\hat{\mathbf{w}}}_K \right\} $ fixed, where $\mathbf{u}\in\mathbb{R}^d$ is a arbitrary unit vector. Notice that
	\begin{equation}\label{A.6}
	{\hat{\mathbf{R}}}_j={\hat{z}}_j{\hat{\mathbf{w}}}_j=0\ \ \ \textup{and}\ \ \ \mathbf{R}_j^\prime={\hat{z}}_j\mathbf{w}_j^\prime=0     
	\end{equation}
	due to ${\hat{z}}_j=0$. Therefore, loss $L$ remains constant under this perturbation, that is, 
	\begin{equation}\label{A.7}
	\begin{split}
	L\left({\hat{z}}_1,{\hat{z}}_2,\cdots{\hat{z}}_K,{\hat{\mathbf{w}}}_1,{\hat{\mathbf{w}}}_2,\cdots\mathbf{w}_j^\prime,\cdots{\hat{\mathbf{w}}}_K\right) \\ =L({\hat{z}}_1,{\hat{z}}_2,\cdots{\hat{z}}_K,{\hat{\mathbf{w}}}_1,{\hat{\mathbf{w}}}_2,\cdots{\hat{\mathbf{w}}}_j,\cdots{\hat{\mathbf{w}}}_K)
	\end{split}
	\end{equation}
	It can be shown that $\left\{{\hat{z}}_1,{\hat{z}}_2,\cdots{\hat{z}}_K,{\hat{\mathbf{w}}}_1,{\hat{\mathbf{w}}}_2,\cdots\mathbf{w}_j^\prime,\cdots{\hat{\mathbf{w}}}_K\right\}$ is also a local minimum of $L$. For any $\mathbf{w}_j$ satisfying $\left \| \mathbf{w}_j-\mathbf{w}_j^\prime \right \|_2\leqslant \frac{\varepsilon }{2}$, there is 
	\begin{center}
		$\left \| \mathbf{w}_j-\hat{\mathbf{w}}_j \right \|_2\leqslant \left \| \mathbf{w}_j-\mathbf{w}_j^\prime \right \|_2+\left \| \mathbf{w}_j^\prime-\hat{\mathbf{w}}_j \right \|_2\leqslant\frac{\varepsilon }{2}+\frac{\varepsilon }{2}= \varepsilon. $
	\end{center}	
	Then by \eqref{A.5} and \eqref{A.7}, we get
	\begin{equation}\label{A.8}
	\begin{split} L({\hat{z}}_1,{\hat{z}}_2,\cdots{\hat{z}}_K,{\hat{\mathbf{w}}}_1,{\hat{\mathbf{w}}}_2,\cdots\mathbf{w}_j,\cdots{\hat{\mathbf{w}}}_K) \\ 
	\geq L({\hat{z}}_1,{\hat{z}}_2,\cdots{\hat{z}}_K,{\hat{\mathbf{w}}}_1,{\hat{\mathbf{w}}}_2,\cdots\mathbf{w}_j^\prime,\cdots{\hat{\mathbf{w}}}_K), 
	\end{split}
	\end{equation}
	which implies that $(\hat{z}_1,\hat{z}_2,\cdots \hat{z}_K,{\hat{\mathbf{w}}}_1,{\hat{\mathbf{w}}}_2,\cdots,\mathbf{w}_j^\prime,\cdots{\hat{\mathbf{w}}}_K)$ is also a local minimum. As a result, similar to \eqref{A.3} we have
	\begin{equation}\label{A.9}
	\frac{\partial L}{\partial\mathbf{R}_j}\left({\hat{\mathbf{R}}}_1,{\hat{\mathbf{R}}}_2,\cdots,\mathbf{R}_j^\prime,\cdots{\hat{\mathbf{R}}}_K\right)\cdot\mathbf{w}_j^\prime=0
	\end{equation}
	Using the fact that ${\hat{\mathbf{R}}}_j=\mathbf{R}_j^\prime=0$ from \eqref{A.6} and consequently $\frac{\partial L}{\partial\mathbf{R}_j}\left({\hat{\mathbf{R}}}_1,{\hat{\mathbf{R}}}_2,\cdots,{\hat{\mathbf{R}}}_\mathbf{j},\cdots{\hat{\mathbf{R}}}_K\right)=\frac{\partial L}{\partial\mathbf{R}_j}\left({\hat{\mathbf{R}}}_1,{\hat{\mathbf{R}}}_2,\cdots,\mathbf{R}_j^\prime,\cdots{\hat{\mathbf{R}}}_K\right)$, subtracting \eqref{A.3} from \eqref{A.9} yields 
	\begin{equation*}
	\frac{\partial L}{\partial\mathbf{R}_j}\left({\hat{\mathbf{R}}}_1,{\hat{\mathbf{R}}}_2,\cdots,{\hat{\mathbf{R}}}_\mathbf{j},\cdots{\hat{\mathbf{R}}}_K\right)\cdot\mathbf{u}=0.
	\end{equation*}
	Since $\mathbf{u}$ is arbitrary, this leads to $\frac{\partial L}{\partial\mathbf{R}_j}\left({\hat{\mathbf{R}}}_1,{\hat{\mathbf{R}}}_2,\cdots,{\hat{\mathbf{R}}}_\mathbf{j},\cdots{\hat{\mathbf{R}}}_K\right)=0$. Therefore, no matter $z_j$ equals 0 or not, we always have $\frac{\partial L}{\partial\mathbf{R}_j}=0\ (j\in[K])$ at local minima. This proof is inspired by \cite{Laurent18}.  \\
\end{proof} 

\begin{lemma}
	$L\left(\mathbf{R}_1,\mathbf{R}_2,\cdots\mathbf{R}_K\right)$ is convex inside cells if $l$ is convex.
\end{lemma}
The convexity of $L\left(\mathbf{R}_1,\mathbf{R}_2,\cdots\mathbf{R}_K\right)$ is proved by showing the positive definiteness of its Hessian. The detailed proof is given in section \ref{section7}.  \\

Now, we are ready to prove Theorem 1.

\begin{proof}[\indent \bfseries Proof of Theorem 1]
	By Lemma 1, inside cells specified by any feasible constant $\left \{ I_{ij},\ i\in [N], j\in [K] \right \} $, the differentiable local minimum of $L(z,\mathbf{w})$ is a stationary point of $L(\mathbf{R})$, which is its unique global minimum due to its convexity. Therefore, differentiable local minima are also global minima for local landscapes inside cells. Furthermore, similar to Lemma 1, one can prove that inside cells local maximum of $L(z,\mathbf{w})$ corresponds to local maximum of $L(\mathbf{R})$. Howerver, the convexity of $L(\mathbf{R})$ means it has no differentiable local maximum. As a result, $L(z,\mathbf{w})$ has no differentiable local maxima.   
\end{proof}

\section{The Locations of Differentiable Local Minima}
\label{section4}
Theorem 1 states that inside cells, all local minima of loss $L(z,\textbf{w})$ are global. In this section, we first find out the locations of $\left\{z_j^\ast,\mathbf{w}_j^\ast,\ j\in [K]\right\}$ that achieve global minima, then give the criteria to judge whether $\left\{\mathbf{w}_j^\ast,\ j\in [K]\right\}$ are inside the defining cells of $\left\{\mathbf{w}_j,\ j\in [K]\right\}$ (we will use the defining cells of $\left\{\mathbf{w}_j,\ j\in [K]\right\}$ and $\left\{\mathbf{w}_j^\ast,\ j\in [K]\right\}$ interchangeably from now on) and consequently truely exist.

\subsection{The Locations and Forms of Differentiable Local Minima}\label{section4.1}
From now on, in order to get analytical solutions we assume that loss function \textit{l} is the squared loss. Lemma 1 implies that for differentiable local minima, there are $\frac{\partial L}{\partial \mathbf{R}_j}=0\ (j\in [K])$, which actually amounts to solving the following least-square problem for constant $\left \{ I_{ij},\ i\in [N], j\in [K] \right \} $,
\begin{equation}\label{eq5}
\mathbf{R}^\ast  = \arg\min_{\mathbf{R}} \frac{1}{N}\sum_{i=1}^{N}(\sum_{j=1}^{K} I_{ij} \mathbf{R}_j\cdot\mathbf{x}_i - y_i)^2.
\end{equation}
The associated linear system $\sum_{j=1}^{K} I_{ij} \mathbf{R}_j\cdot\mathbf{x}_i = y_i \ (i\in [N])$ can be rewritten in the following form 
\begin{equation}\label{eq6}
A\mathbf{R}=\mathbf{y},A=\left(\begin{matrix}I_{11}\mathbf{x}_1^T&\cdots& I_{1K}\mathbf{x}_1^T\\\vdots&\ddots&\vdots\\I_{N1}\mathbf{x}_N^T&\cdots& I_{NK}\mathbf{x}_N^T\\\end{matrix}\right),\mathbf{y}=\left(\begin{matrix}\begin{matrix}y_1\\y_2\\\end{matrix}\\\begin{matrix}\vdots\\y_N\\\end{matrix}\\\end{matrix}\right),
\end{equation}
where $ \mathbf{R}=\begin{pmatrix}
{\mathbf{R}_1}^T& \hdots &{\mathbf{R}_K}^T 
\end{pmatrix}^T$. Here we have changed the meaning of $\mathbf{R}$ from a set in \eqref{eq2} to a vector without hampering the understanding. According to matrix theory (see subsection \ref{section2.2}), the general solution $\mathbf{R^{\ast }}$ to the least square problem \eqref{eq5} can be expressed as follows using the Moore-Penrose inverse of $ A\in\ \mathbb{R}^{N\times Kd}$, 
\begin{equation}\label{eq7}
\mathbf{R^{\ast }}=A^+\mathbf{y} + \left(I-A^+A\right)\mathbf{c},
\end{equation}
where $ \mathbf{c}\in \mathbb{R}^{Kd}$ is a arbitrary vector, $I$ is identity matrix.

The optimal solution $\mathbf{R}^\ast$ can be characterized by the following cases:

1). $\mathbf{R}^\ast$ is unique: $\mathbf{R}^\ast = A^+\mathbf{y}$, corresponding to $A^+A=I$ and thus $(I-A^+A)\mathbf{c}$ vanishes. This happens if and only if $rank (A)=Kd$. Therefore, $N\geq Kd$ is necessary in order to have a unique solution. Using the full-rank decomposition of $ A $ when it has full rank, we have $A^+={(A^TA)}^{-1}A^T$, the solution can then be written as 
\begin{equation}\label{eq8}
\mathbf{R}^\ast={(A^TA)}^{-1}A^T\mathbf{y}. 
\end{equation}
\eqref{eq8} can also be obtained by solving the linear system resulted from $\frac{\partial L}{\partial\mathbf{R}_j}=0\ (j\in [K])$, an approach we will take to deal with saddle points in section \ref{section5}. 	 

2). $\mathbf{R}^\ast$ has infinite number of continuous solutions. In this case, $I-A^+A\neq0$, hence the arbitrary vector $\mathbf{c}$ plays a role. This happens only if $rank(A)\neq Kd$. As a result, there are two possible situations in which infinite number of optimal solutions exist. a). $N<Kd$. This is usually refered to as over-parameterization and $A\mathbf{R}=\mathbf{y}$ has infinite number of solutions if $rank(A)=rank(A, \ \mathbf{y})$, with some components of $\mathbf{R}$ being free variables. b). $N\geq Kd$ but $rank(A)< Kd$. One example is that some hidden neurons are not activated by all samples ( i.e., $\forall i\in [N], I_{ij}=0$. The corresponding columns in \textit{A} are zeros). $\mathbf{R}_j$ associated with such hidden neuron does not affect loss \textit{L}, hence can be changed freely. An extreme is that all hidden units are not activated by any sample, leading to $A=0$ and consequently $A^+=0$ and $\mathbf{R}^\ast=\mathbf{c}$. In this case, $\mathbf{R}^\ast$ can be any point in the whole weight space, and thus the local landscape is a flat plateau.

In general, \eqref{eq7} shows $\mathbf{R}^\ast$ is a affine transformation of $ \mathbf{c}\in \mathbb{R}^{Kd}$. Therefore, $\mathbf{R}_j^\ast$ can be a isolated point, the whole $\mathbb{R}^{d}$ space or a linear subspace (a hyperplane) of it, depending on whether $rank (A)=Kd$ and the rows in $\left(I-A^+A\right)$ corresponding to $\mathbf{R}_j^\ast$ is of full rank or not. Since $ A $ is specified by data and activation pattern $I_{ij}$, the form of differentiable local minima is jointly determined by data, activation pattern of hidden neurons and network size ($K$ and $d$).

To get the loss at these minima, we substitute \eqref{eq7} into $L(\mathbf{R}^\ast)=\frac{1}{N}\left \| A\mathbf{R}^\ast-\mathbf{y} \right \|_{2}^{2}$ and get
\begin{equation}\label{eq9}
L(\mathbf{R}^\ast)=\frac{1}{N}\left \| AA^+\mathbf{y}-\mathbf{y} \right \|_{2}^{2}
\end{equation}
The loss $ L $ will be zero only if $ AA^+\mathbf{y}=\mathbf{y}$, corresponding to that the original linear system $A\mathbf{R}=\mathbf{y}$ is solvable (see subsection \ref{section2.2}).

\subsection{An Illustrative Example}\label{section4.2}

We give a simple example to illustrate different cases of differential local minima. Suppose there is only one hidden neuron, and there are two samples in two-dimensional input space: $\mathbf{x}_1=\begin{pmatrix}
1 & 0
\end{pmatrix}^T,\ \mathbf{x}_2=\begin{pmatrix}
0 & 1
\end{pmatrix}^T$ with labels $y_1=1,\ y_2=1$. We set $z=1$ and bias $b=0$. Denoting the only weight vector as $\mathbf{w}$, the two samples then become two lines in the space of $\mathbf{w}$, and their normal vectors are shown in Fig.2(a). There are in total four cells in the $\mathbf{w}$ space. Fig.2(b) shows the global landscape, from which one can see that there are no spurious differential local minima in each cell, and the differential local minima are either a single point, a line or a flat plateau. Fig.2(b) also exhibits that although the continuous local minima in cells $r_1$, $r_2$ and $r_3$ are global minima in corresponding cells, they are still bad minima with respect to the global landscape. 

In cell $r_1, I_{11}=I_{21}=0$, thus $A=0$ and $\mathbf{R}^\ast$ is arbitrary. According to \eqref{eq9}, the loss $L=\frac{1}{2}\left(y_1^2+y_2^2\right)=1$. Actually, in cell $r_1$, both samples are not activated and the loss does not change with $\mathbf{w}$, thus the local landscape is a flat plateau. Cell $r_2$ and $r_3$ are similar, and we will take $r_3$ as an example. In $r_3$, $I_{11}=1,{\ I}_{21}=0$, hence
$A=\left(\begin{matrix}\mathbf{x}_1^T\\0\\\end{matrix}\right)=\left(\begin{matrix}1&0\\0&0\\\end{matrix}\right)$, $A^+=\left(\begin{matrix}1&0\\0&0\\\end{matrix}\right)$, $\mathbf{R}^\ast=A^+\mathbf{y}+\left(I-A^+A\right)\mathbf{c}=\left(\begin{matrix}1&0\\0&0\\\end{matrix}\right)\left(\begin{matrix}1\\1\\\end{matrix}\right)+\left(\begin{matrix}0&0\\0&1\\\end{matrix}\right)\left(\begin{matrix}c_1\\c_2\\\end{matrix}\right)=\left(\begin{matrix}1\\c_2\\\end{matrix}\right)$, which is a line with distance 1 to $\mathbf{x}_1$. The minimal loss in $r_3$ is $L=\frac{1}{2}$. In region $r_4$, $I_{11}={\ I}_{21}=1,\ A=\left(\begin{matrix}\mathbf{x}_1^T\\\mathbf{x}_2^T\\\end{matrix}\right)=\left(\begin{matrix}1&0\\0&1\\\end{matrix}\right)$, thus $A^+=\left(\begin{matrix}1&0\\0&1\\\end{matrix}\right), \mathbf{R}^\ast=\mathbf{y}=\left(\begin{matrix}1\\1\\\end{matrix}\right)$, indicating the landscape in $r_4$ has a unique minimum. The minimal loss in $r_4$ is $L=0$ by \eqref{eq9}, hence the local minimum in $r_4$ is the global minimum of whole landscape.


\begin{figure*}
	\centering
	\subfigure[four cells in weight space.]{
		\label{fig:subfig:onefunction} 
		\includegraphics[width=4.5cm]{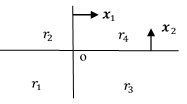}}
	\hspace{0.3cm}
	\subfigure[loss landscape for $y_1=1, y_2=1$.]{
		\label{fig:subfig:twofunction} 
		\includegraphics[width=4.0cm]{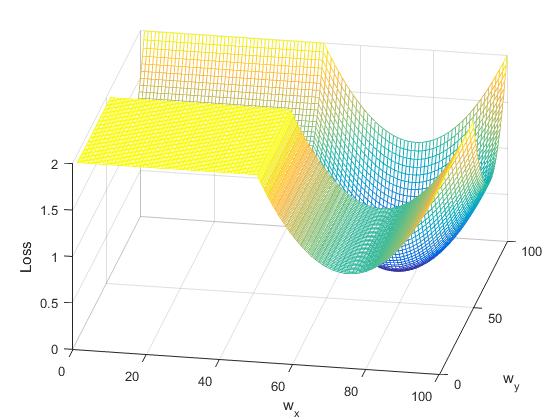}}
	\hspace{0.3cm}
	\subfigure[loss landscape for $y_1=1, y_2=-1$.]{
		\label{fig:subfig:twofunction} 
		\includegraphics[width=4.0cm]{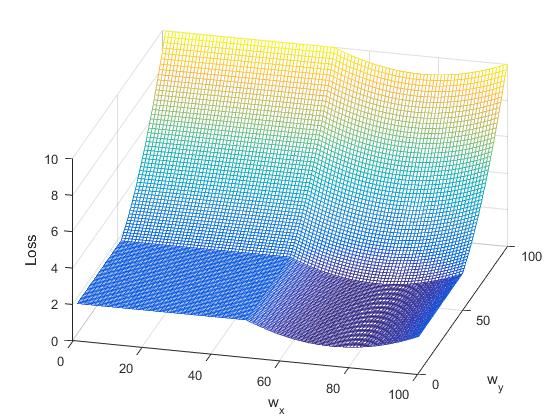}}
	\hspace{0.3cm}
	\caption{An illustrative example of loss landscape with two samples.}
	\label{fig:twopicture} 
\end{figure*}

\subsection{Criteria for Existence of Genuine Differentiable Local Minima}\label{section4.3}

In the above example, if $y_2=-1$, $\mathbf{R}^\ast$ for cell $r_4$ will be $\begin{pmatrix} 1 &-1 \end{pmatrix}^T$, which is actually outside $r_4$. In this situation, the local landscape of $r_4$ has no differential local minima at all, as shown in Fig.2(c). In this subsection, we are going to present conditions under which $\mathbf{R}^\ast$ will be inside their defining cells and we call such local minima as \textbf{genuine differentiable local minima}. In line with different cases of $\mathbf{R}^\ast$, the criteria for each case are discussed as follows.

1). For the case $\mathbf{R}^\ast$ is unique, in order for $\mathbf{w}^\ast$ to be inside the defining cells, $\mathbf{w}^\ast$ and $\mathbf{w}$ should be on the same side of each sample. Giving $\left \{ I_{ij}, i\in [N], j\in [K] \right \} $ that specify the defining cells, this can be expressed as
\begin{equation}\label{eq10}
\mathbf{w}_j^\ast\cdot\mathbf{x}_i \begin{cases} >0 \ \ \text{if } I_{ij}=1; \\ \le0 \ \ \text{if } I_{ij}=0; \end{cases} \ (i\in [N]; j\in [K]).
\end{equation}
Since $\mathbf{R}_j^\ast=z_j^\ast\mathbf{w}_j^\ast$, the conditions are transformed into $\frac{1}{z_j^\ast} \mathbf{R}_j^\ast \cdot \mathbf{x}_i{_<^>}\ 0$. Except for its sign, the magnitude of $z_j^\ast$ does not affect the conditions, and consequently for given $\mathbf{R}_j^\ast$ the differentiable local minima $(z_j^\ast, \mathbf{w}_j^\ast)$ have two branches, corresponding to different signs of $z_j^\ast$. As a result, the criteria for existence of unique differentiable local minima can be expressed as: for each $\mathbf{R}_j^\ast\, (j\in [K])$,
\begin{equation}\label{eq11}
\mathbf{R}_j^\ast\cdot\mathbf{x}_i \begin{cases} >0 \ \ \text{if } I_{ij}=1; \\ \le0 \ \ \text{if } I_{ij}=0; \end{cases} \ (i\in [N])  
\end{equation}
\begin{equation}\label{eq12}
\ \ or \ \  \mathbf{R}_j^\ast\cdot\mathbf{x}_i \begin{cases} <0 \ \ \text{if } I_{ij}=1; \\ \ge0 \ \ \text{if } I_{ij}=0; \end{cases} \ (i\in [N])
\end{equation}
2). For the case $\mathbf{R}^\ast$ is continuous, we need to test whether the continuous differentiable local minima in \eqref{eq7} are in their defining cells. For example, substituting \eqref{eq7} into \eqref{eq11}, then for each $\mathbf{R}_j^\ast\ (j\in [K])$ the criteria become
\begin{gather}\label{eq13}
\mathbf{x}_i^T((A^+\mathbf{y})_j+(I-A^+A)_j\mathbf{c}) \begin{cases} >0 \ \ \text{if } I_{ij}=1; \\ \le0 \ \ \text{if } I_{ij}=0; \end{cases}  (i\in [N])
\end{gather}
where ${(A^+\mathbf{y})}_j$ is the rows of $A^+\mathbf{y}$ corresponding to $\mathbf{R}_j^\ast$, and so on. Each inequality of $\mathbf{c}$ in \eqref{eq13} defines a half-space in $\mathbb{R}^{Kd}$. Therefore, the criteria for existing genuine continuous differentiable local minima are reduced to identifying whether the intersection of all these half-spaces is null. The intersection, if not null, will be a convex high-dimensional polyhedron. Efficient implementation of half-spaces intersection to judge the existence of differentiable local minima will be discussed in part 2 of this work \cite{globallosslandscape_part2}.

\section{Saddle Points}\label{section5}
In this section, we will study the existence of differentiable saddle points that are located inside cells.
\subsection{Necessary and Sufficient Conditions for Existence of Differentiable Saddle Points}\label{section5.1}
Unlike local minima, saddle points are stationary points that have both ascent and descent directions in their neighborhood, thus their Hessians are indefinite. 

The necessary and sufficient conditions for existence of differentiable saddle points are given in the following theorem.

\begin{thm}	
	For the loss in \eqref{eq1} with \textit{l} being the squared loss, there exist differentiable saddle points for all combinations of the form $ (j_1,j_2,\cdots,j_{K^\prime};\ K^\prime<K) $, where $(j_1,j_2,\cdots,j_K)$ is any permutation of $(1,2,\cdots,K)$. Optimal $ \left \{  \mathbf{R}_j^\ast,\ (j=j_1,j_2,\cdots,j_{K^\prime}) \right \} $ of saddle points are the solutions to the linear system $B\widetilde{\mathbf{R}}=\mathbf{b}$, i.e., 
	\begin{equation}\label{eq51}
	\widetilde{\mathbf{R}}^\ast=B^+\mathbf{b} + \left(I-B^+B\right)\mathbf{c}, \ \  \mathbf{c} \ \textup{is arbitrary}. 
	\end{equation}
	where
	$\widetilde{\mathbf{R}}={(\mathbf{R}_{j_1}^T,\mathbf{R}_{j_2}^T,\cdots,\mathbf{R}_{j_{K^\prime}}^T)}^T$, $B\in \mathbb{R}^{{K^\prime}d\times {K^\prime}d}$ is a block matrix and $\mathbf{b}\in \mathbb{R}^{{K^\prime}d}$ is a block vector with the following components, 
	\begin{gather}\label{eq14}
	B(j,k)=\sum_{i=1}^{N}{I_{ij}\mathbf{x}_i\cdot} I_{ik}\mathbf{x}_i^T \ , \notag  \\ \mathbf{b}(j)=\sum_{i=1}^{N}{I_{ij}\cdot} y_i\mathbf{x}_i,\  (j,k=j_1,j_2,\cdots,j_{K^\prime}).
	\end{gather}
	Optimal $ \left \{  \mathbf{R}_j^\ast,\ (j=j_{K^\prime+1},\cdots,j_K) \right \} $ of saddle points satisfy
	\begin{equation}\label{eq15}
	\sum_{i=1}^{N}e_i I_{ij}\mathbf{x}_i \cdot \mathbf{w}_j^\ast =0, 
	\end{equation}
	where the error $e_i=\sum_{k=j_1}^{j_{K^\prime}}{I_{ik}\mathbf{R}_k^\ast}\cdot\mathbf{x}_i-y_i$. \eqref{eq51} and \eqref{eq15} are both necessary and sufficient for $\left \{z_j^\ast,\mathbf{w}_j^\ast \right \} $ to be saddle points.
	
\end{thm}

\begin{proof}[\indent \bfseries Proof]
	Since saddle points are stationary points, \eqref{A.3} and \eqref{A.4} still hold. If $\frac{\partial L}{\partial\mathbf{R}_j}=0$, \eqref{A.3} and \eqref{A.4} are both satisfied. On the other hand, $z_j^\ast=0$ if $\frac{\partial L}{\partial\mathbf{R}_j}\neq0$ by \eqref{A.4}. However, $\frac{\partial L}{\partial\mathbf{R}_j}(j=1,2,\cdots,K)$ can not all equal zero at the same time, otherwise the solutions would be differentiable local minima rather than saddle points. Without loss of generality, suppose $\frac{\partial L}{\partial\mathbf{R}_j}=0\ (j=j_1,j_2,\cdots,j_{K^\prime};K^\prime<K)$, and the remaining $\frac{\partial L}{\partial\mathbf{R}_j}\ (j=j_{K^\prime+1},\cdots,j_K)$ are non-zeros. We need to test all possible combinations of the form $\left(j_1,j_2,\cdots,j_{K^\prime}\right)$ such that $ \frac{\partial L}{\partial\mathbf{R}_j}=0\ (j=j_1,j_2,\cdots,j_{K^\prime};K^\prime<K)$, and see whether there exist saddle points.  \\
	
	Ignoring the factor $\frac{1}{N} $ in $L$ and $\frac{2}{N} $ in $ \frac{\partial L}{\partial\mathbf{R}_j} $ from now on, we have	
	\begin{equation}\label{C.5}
	\begin{split}
	\frac{\partial L}{\partial\mathbf{R}_j}=0=\sum_{i=1}^{N}{\left(\sum_{k=1}^{K}{I_{ik}\mathbf{R}_k^\ast\cdot\mathbf{x}_i}-y_i\right)\cdot I_{ij}\mathbf{x}_i}\ , \\ j=j_1,j_2,\cdots,j_{K^\prime} .     
	\end{split}	
	\end{equation}	
	Since $\frac{\partial L}{\partial\mathbf{R}_k}\neq0\ (k=j_{K^\prime+1},\cdots,j_K),$ by $\mathbf{R}_k^\ast=0$ due to associated $z_k^\ast=0$, we get	
	\begin{equation}\label{C.6}
	\begin{split}
	\sum_{i=1}^{N}{\left(\sum_{k=j_1}^{j_{K^\prime}}{I_{ik}{\mathbf{x}_i^T} \mathbf{R}_k^\ast-y_i}\right)\cdot I_{ij}\mathbf{x}_i}\ =0, \\  j=j_1,j_2, \cdots, j_{K^\prime}.
	\end{split}	
	\end{equation}	
	Let $\widetilde{\mathbf{R}}={(\mathbf{R}_{j_1}^T,\mathbf{R}_{j_2}^T,\cdots,\mathbf{R}_{j_{K^\prime}}^T)}^T$, \eqref{C.6} leads to the following linear system 	
	\begin{equation*}
	\mathit{B}\widetilde{\mathbf{R}}=\mathbf{b}
	\end{equation*}                             	
	where $\mathit{B}\in\mathbb{R}^{K^\prime d\times K^\prime d}$ is a block matrix and $\mathbf{b}\in\mathbb{R}^{K^\prime d}$ is a block vector with components as shown in the theorem. The linear system is solvable if and only if $BB^+\mathbf{b}=\mathbf{b}$. Using the facts that $ B=A^TA $ and $ \mathbf{b}= A^T\mathbf{y}$, where $ A $ has the same form as that in \eqref{eq6} but with only $ K^\prime d $ columns, , and the properties of Moore-Penrose inverse described in subsection \ref{section2.2}, we can prove $BB^+\mathbf{b}=\mathbf{b}$. Actually, we have $ BB^+\mathbf{b} = (A^TA)(A^TA)^+A^T\mathbf{y} = (A^TA)A^+\mathbf{y} =  A^T(AA^+)^T \mathbf{y} =  A^T (A^T)^+ A^T \mathbf{y} = A^T \mathbf{y} = \mathbf{b}$, thus $ \mathit{B}\widetilde{\mathbf{R}}=\mathbf{b} $ is always solvable, and the solution $ \widetilde{\mathbf{R}}^\ast $ is given in \eqref{eq51}.
	
	${\widetilde{\mathbf{R}}}^\ast$ can be a single point, the whole $\mathbb{R}^{K^\prime d}$ sapce or a linear subspace in $\mathbb{R}^{K^\prime d}$, corresponding to $rank(B)=\ K^\prime d$, $\left(I-B^+B\right)$ is of full rank or not respectively. 
	
	For $j=j_{K^\prime+1},\cdots,j_K$ with $\frac{\partial L}{\partial\mathbf{R}_j}\neq0, $ \eqref{A.3} should be satisfied, resulting in 	
	\begin{equation*}
	\sum_{i=1}^\mathit{N}[(\sum_{k=j_1}^{j_{K^\prime}}{I_{ik}\mathbf{R}_k^\ast}{\cdot\mathbf{x}}_i-y_i)\mathit{I}_{ij}\mathbf{x}_i]\cdot\mathbf{w}_j^\ast=0
	\end{equation*}	
	Defining error $e_i=\sum_{k=j_1}^{j_{K^\prime}}{I_{ik}\mathbf{R}_k^\ast}{\cdot\mathbf{x}}_i-y_i$,
	we have 	
	\begin{equation*}	[\sum_{i=1}^{N}\mathit{e}_i\mathit{I}_{ij}\mathbf{x}_i]\cdot\mathbf{w}_j^\ast=0, \quad j=j_{K^\prime+1},\cdots,j_K
	\end{equation*}   	
	Therefore, $\mathbf{w}_j^\ast$ is on a hyperplane that passes the origin in the space of $\mathbf{w}$. \eqref{eq51} and \eqref{eq15} constitute the necessary conditions that saddle points $\left \{z_j^\ast,\mathbf{w}_j^\ast \right \} $ must satisfy. \\
	
	Now we proceed to prove that \eqref{eq51} and \eqref{eq15} are also sufficient for the existence of saddle points. Our approach is to prove that there exist both ascent and descent directions at points $\left \{z_j^\ast,\mathbf{w}_j^\ast \right \} $ found by \eqref{eq51} and \eqref{eq15}. For any $ k $ such that $\frac{\partial L}{\partial\mathbf{R}_k}\neq0\ ($thus$\ z_k^\ast=0)$, we perturbate $z_k^\ast$ and $\mathbf{w}_k^\ast$ respectively as follows: $0\rightarrow\delta z_k,\ \mathbf{w}_k^\ast\rightarrow\mathbf{w}_k^\ast+\delta\mathbf{w}_k$. The loss function $\mathit{L}$ after perturbation is
	\begin{equation} \label{C.8}
	\begin{split}
	L^\prime&=\sum_{i=1}^N[\sum_{j=j_1}^{j_{K^\prime}}I_{ij}\mathbf{R}_j^\ast\cdot\mathbf{x}_i+I_{ik}\delta z_k\cdot\left(\mathbf{w}_k^\ast+\delta\mathbf{w}_k\right)\cdot\mathbf{x}_i-y_i]^2 \\
	&= \sum_{i=1}^N[e_i+I_{ik}\delta z_k\cdot\left(\mathbf{w}_k^\ast+\delta\mathbf{w}_k\right){\cdot\mathbf{x}}_i]^2 \\
	&= L+2\sum_{i=1}^{N}{e_i}I_{ik}\mathbf{w}_k^\ast \cdot\mathbf{x}_i\delta z_k+2\sum_{i=1}^{N}	{{e_i}I_{ik}\delta z_k\delta\mathbf{w}_k}\cdot\mathbf{x}_i \\ & \ \ +\sum_{i=1}^{N}{I_{ik}\delta z_k^2}	{(\mathbf{w}_k^\ast\cdot\mathbf{x}_i)}^2, 
	\end{split}
	\end{equation}
	where we have used $I_{ik}^2=I_{ik}$ and ignored terms higher than 2nd-order. Applying \eqref{eq15}, we get
	\begin{equation}\label{C.9}
	\bigtriangleup L=L^\prime - L=2\sum_{i=1}^Ne_iI_{ik}{\delta z_k}{\delta \mathbf{w}_k}\cdot\mathbf{x}_i+
	\sum_{i=1}^{N}I_{ik}{\delta {z_k}^2}\left(\mathbf{w}_k^\ast\cdot\mathbf{x}_i\right)^2
	\end{equation}	
	Only 2nd-order terms remain in \eqref{C.9}. If $\delta z_k$ is very small and $\delta\mathbf{w}_k$ not too small, we only need to consider the term $\sum_{i=1}^{N}{e_iI_{ik}\delta z_k\delta\mathbf{w}_k}{\cdot\mathbf{x}}_i$. Notice that $I_{ik}$ can not be zero for all $i\in[N]$, otherwise $\frac{\partial L}{\partial\mathbf{R}_k}=\sum_{i=1}^{N}{\left(\sum_{j=1}^{K}{I_{ij}\mathbf{R}_j\cdot\mathbf{x}_i}-y_i\right)\cdot I_{ik}\mathbf{x}_i}=0$, contradicting our assumption that $\frac{\partial L}{\partial\mathbf{R}_k}\neq0$. Therefore, setting $\delta z_k>0$, we can make $\bigtriangleup{L}_<^> 0$ by setting $(\sum_{i=1}^{N}{e_iI_{ik}}\mathbf{x}_i)\cdot\delta\mathbf{w}_k{_<^>}\ 0$ with appropriate $\delta\mathbf{w}_k$, indicating both ascent and descent directions exist. Therefore, $\left \{z_j^\ast,\mathbf{w}_j^\ast \right \} $ found by \eqref{eq51} and \eqref{eq15} are saddle points.                   
\end{proof}

\subsection{ Conditions for Existence of Genuine Differentiable Saddle Points}
\label{section5.2}
Like differentiable local minima, differentiable saddle points found by \eqref{eq51} and \eqref{eq15} may be outside their defining cells. The criteria for existence of genuine saddle points can be derived in a similar way as those for differentiable loacl minima. The main difference with differentiable local minima is that although ${\widetilde{\mathbf{R}}}^\ast$ can be a single point, whole space or a linear subspace, $\mathbf{w}_j^\ast\ (j=j_{K^\prime+1},\cdots,j_K)$ are on hyperplanes and hence one need to test their intersections with corresponding defining cells. Only when all $\mathbf{w}_j^\ast\ (j=1,2,\cdots,K)$ are inside their defining cells, there exist genuine differentiable saddle points.

\section{Non-differentiable Local Minima}\label{section6}
After understanding the local landscapes inside cells, we now turn our focus to local minima that lie on cell boundaries. 

We consider the case in which a weight vector lies on the boundary of two cells and thus the loss function \textit{L} in \eqref{eq2} is non-differentiable. Suppose $\mathbf{w}_m$ is located on the boundary of cell 1 and cell 2, separated by a sample $\mathbf{x}_n$, see Fig.3. We are going to give the necessary and sufficient conditions for $\left\lbrace z_j\  (j\in [K]);\ \mathbf{w}_j\ (j\in [K],j\neq m),\mathbf{w}_m \right\rbrace $ to be a non-differentiable local minimum.

$\frac{\partial L}{\partial\mathbf{w}_m}$ is non-differentiable and may be not equal to zero. In the following lemma, we first give the constraints on $\frac{\partial L}{\partial\mathbf{w}_m}$ in order for $\mathbf{w}_m$ to be part of a local minimum.

\begin{lemma}
	Suppose ${\widetilde {\mathbf{w}}}_m$ lies on the boundary of cell 1 and cell 2 seperated by a sample $\mathbf{x}_n$, where cell 2 is on the positive side of $\mathbf{x}_n$ and cell 1 on the negative side. ${\widetilde{\mathbf{w}}}_m$ is a on non-differentiable minimum if and only if 
	\begin{equation}\label{D.1}
	\left(\lim_{\mathbf{w}_m\to{\widetilde{\mathbf{w}}}_m}{\frac{\partial L}{\partial\mathbf{w}_m}|_1}\right)//(-\mathbf{x}_n)\ \ and\ \ \left(\lim_{\mathbf{w}_m\to{\widetilde{\mathbf{w}}}_m}{\frac{\partial L}{\partial\mathbf{w}_m}|_2}\right)//\mathbf{x}_n ,              
	\end{equation}
	where $\mathbf{a}//\mathbf{b}$ denotes vectors $\mathbf{a}$ and $\mathbf{b}$ are in the same direction, $\frac{\partial L}{\partial\mathbf{w}_m}|_1$ means $\frac{\partial L}{\partial\mathbf{w}_m}$ in cell 1.
\end{lemma} 

In other words, at non-differentiable local minima $\frac{\partial L}{\partial\mathbf{w}_m}|_1$ and $\frac{\partial L}{\partial\mathbf{w}_m}|_2$ are perpendicular to the hyperplane of $\mathbf{x}_n$ and have opposite directions. Since $ L $ is indifferentiable w.r.t. $\mathbf{w}_m \ at \ {\widetilde{\mathbf{w}}}_m$, we use the limit.
\begin{figure}[ht]
	\vskip 0.2in
	\begin{center}
		\centerline{\includegraphics[width=5.0cm]{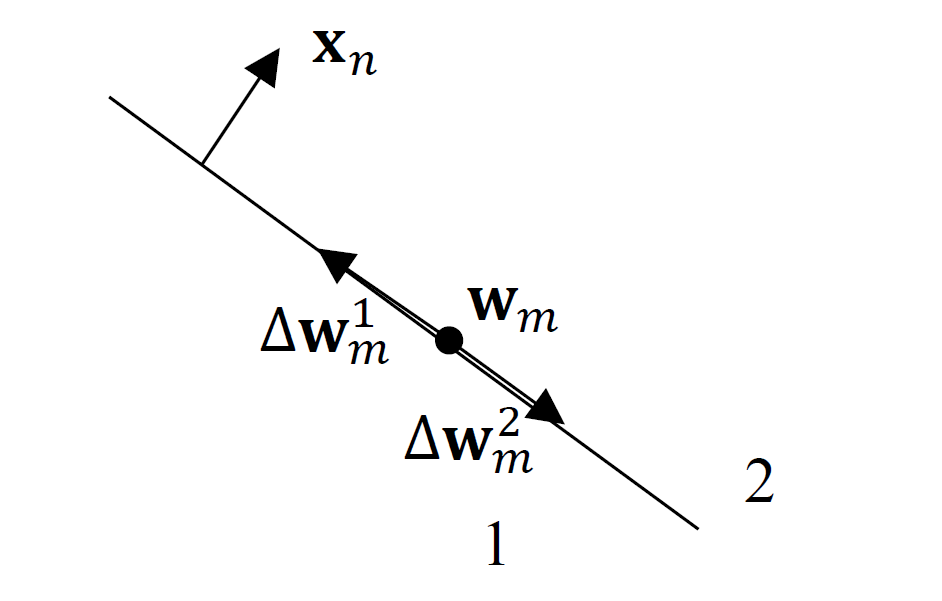}}
		\caption{A non-differentiable local minimum $\mathbf{w}_m$ lying on the cell boundary defined by a sample $\mathbf{x}_n$.}
		\label{one pic}
	\end{center}
	\vskip -0.2in
\end{figure}

\begin{proof}[\indent \bfseries Proof of Lemma 3]
	By 1st-order Taylor expanssion, $L\left(\widetilde{\mathbf{w}}_m+\bigtriangleup\mathbf{w}_m\right)=L\left(\widetilde{\mathbf{w}}_m\right)+\frac{\partial L}{\partial\mathbf{w}_m}({\widetilde{\mathbf{w}}}_m)\cdot\bigtriangleup\mathbf{w}_m$. Here we omit other variables in $L$ and only perturbate $\mathbf{w}_m$. If ${\widetilde{\mathbf{w}}}_m$ is on a local minimum, any perturbation $\bigtriangleup\mathbf{w}_m$ should not cause $L$ to decrease, i.e., $\frac{\partial L}{\partial\mathbf{w}_m}({\widetilde{\mathbf{w}}}_m)\cdot\bigtriangleup\mathbf{w}_m\geq 0.$ If $\frac{\partial L}{\partial\mathbf{w}_m}|_2\left(\widetilde{\mathbf{w}}_m\right):= \lim_{\mathbf{w}_m\to{\widetilde{\mathbf{w}}}_m}{\frac{\partial L}{\partial\mathbf{w}_m}|_2}$ 
	is not in the direction of $\mathbf{x}_n$, one can always find $\bigtriangleup\mathbf{w}_m$ such that $\frac{\partial L}{\partial\mathbf{w}_m}|_2\left({\widetilde{\mathbf{w}}}_m\right)\cdot\bigtriangleup\mathbf{w}_m<0$, such as either $\bigtriangleup\mathbf{w}_m^1$ or $\bigtriangleup\mathbf{w}_m^2$ in Fig. 3, indicating descent directions exist in cell 2 and contradicting the assumption that ${\widetilde{\mathbf{w}}}_m$ is on a local minimum. Therefore, we have $\frac{\partial L}{\partial\mathbf{w}_m}|_2//\mathbf{x}_n.$ $\ \frac{\partial L}{\partial\mathbf{w}_m}|_1//(-\mathbf{x}_n)$ can be proved in a similar way.  
	
	On the other hand, if \eqref{D.1} holds, any $\bigtriangleup\mathbf{w}_m$ will increase or keep the loss by 
	$\frac{\partial L}{\partial\mathbf{w}_m}\left(\widetilde{\mathbf{w}}_m\right) \cdot \bigtriangleup\mathbf{w}_m\geq0.$ 
	Therefore, \eqref{D.1} is sufficient for ${\widetilde{\mathbf{w}}}_m$ to be on a local minimum.                        
\end{proof}                                                          

The conditions for existence of non-differentiable local minima are given by the following theorem.

\begin{thm}
	For the loss in \eqref{eq1} with l being the squared loss, there exist non-differentiable local minima $\left\lbrace z^\ast_j\  (j\in [K]);\ \mathbf{w}^\ast_j\ (j\in [K],j\neq m),\mathbf{w}^\ast_m \right\rbrace $, where $ \mathbf{w}^\ast_m $ is located on the boundary of two cells seperated by a sample $\mathbf{x}_n$, if and only if the linear system D$\mathbf{R}=\mathbf{d}$ is solvable, where $\mathbf{R}=\left(\mathbf{R}_1^T,\mathbf{R}_2^T,\cdots,\mathbf{R}_K^T\right)^T,\ \ D\in\mathbb{R}^{(K+1)d\times K d}$ is a matrix with the following block components,
	\begin{equation}\label{D.2}
	\begin{split}
	&D\left(j,k\right)=\sum_{i} I_{ij}I_{ik}\mathbf{x}_i\mathbf{x}_i^T\quad \left(j,k\in[K]; j\neq m\right), \\
	&D\left(m,k\right)=\sum_{i\neq n} I_{im}I_{ik}{(\mathbf{x}}_i\cdot \mathbf{x}_n\mathbf{x}_n-\left|	\mathbf{x}_n\right|^2\mathbf{x}_i)\mathbf{x}_i^T\quad (k\in[K]) \\
	&D\left(K+1,m\right)=\mathbf{x}_n^T ,\ \ \ \  D\left(K+1,k\right)=\mathbf{0}\ (k\in[K]; k\neq m),       
	\end{split}
	\end{equation}
	and $\mathbf{d}{\in\mathbb{R}}^{(K+1)d}$ is a block vector with the following block components,
	\begin{equation}\label{D.3}
	\begin{split}
	&\mathbf{d}\left(j\right)=\sum_i{I_{ij}y_i\mathbf{x}_i}\ \ \ \ (j\in[K]; j\neq m) \\
	&\mathbf{d}\left(m\right)=\sum_{i\neq n}{I_{im}y_i\left(\mathbf{x}_i\cdot\mathbf{x}_n\mathbf{x}_n-	\left|\mathbf{x}_n\right|^2\mathbf{x}_i\right)},\\ 
	&\mathbf{d}\left(K+1\right)=0,    
	\end{split}
	\end{equation}
	and its solution $\mathbf{R}^\ast$ satisfies the following two inequalities for either $z_m>0$ or $z_m<0$,
	\begin{equation}\label{D.4}
	\ \ \sum_{i\neq n}{\left[(\sum_{k}{I_{ik}\mathbf{R}^\ast_k\cdot\mathbf{x}_i-y_i)I_{im}\mathbf{x}_i\cdot\mathbf{x}_n}\right]z_m<0},                 
	\end{equation}
	\begin{equation}\label{D.5}
	\begin{split}
	&\sum_{i\neq n}{\left[(\sum_{k}{I_{ik}\mathbf{R}^\ast_k\cdot\mathbf{x}_i-y_i)I_{im}\mathbf{x}_i\cdot\mathbf{x}_n}\right]z_m} \\ +&\left[(\sum_{k}{I_{nk}\mathbf{R}^\ast_k\cdot\mathbf{x}_n-y_n)\left|\mathbf{x}_n\right|^2}\right]z_m>0 .  
	\end{split}
	\end{equation}
\end{thm}

\begin{proof}[\indent \bfseries Proof]
	At non-differentiable local minima, we have
	\begin{equation}\label{D.6}
	\frac{\partial L}{\partial z_j}=\frac{\partial L}{\partial\mathbf{w}_j}=0\ \ \ (j\in[K];j \neq m)
	\end{equation}
	\begin{equation}\label{D.7}
	\frac{\partial L}{\partial z_m}=0,
	\end{equation}	
	due to these derivatives are differentiable. Similar to Lemma 1, \eqref{D.6} leads to
	\begin{equation}\label{D.8}
	\frac{\partial L}{\partial\mathbf{R}_j}=0 \ \ (j \in [K];j\neq m).
	\end{equation}
	$\mathbf{w}_m$ is on the hyperplane of $\mathbf{x}_n$ means
	\begin{equation}\label{D.9}
	\mathbf{w}_m\cdot\mathbf{x}_n=0.
	\end{equation}	
	\eqref{D.7},\eqref{D.8},\eqref{D.1} and \eqref{D.9} constitute the necessary and sufficient conditions for non-differentiable local minima. We now write them in detailed forms. \\
	
	First, the derivatives $\frac{\partial L}{\partial\mathbf{w}_m}\mid_1 $ and $\frac{\partial L}{\partial\mathbf{w}_m}\mid_2$ are 
	\begin{equation}\label{D.10}
	\frac{\partial L}{\partial\mathbf{w}_m}\mid_1=\sum_{i\neq n}\left[(\sum_{k}{I_{ik}\mathbf{R}_k\cdot\mathbf{x}_i-y_i)I_{im}\mathbf{x}_i}\right]\cdot z_m
	\end{equation}
	\begin{equation}\label{D.11}
	\frac{\partial L}{\partial\mathbf{w}_m}\mid_2=\frac{\partial L}{\partial\mathbf{w}_m}|_1+\left[(\sum_{k}{I_{nk}\mathbf{R}_k\cdot\mathbf{x}_n-y_n)\mathbf{x}_n}\right]\cdot z_m   
	\end{equation}
	\eqref{D.7} yields
	\begin{equation}\label{D.12}
	\sum_{i=1}^{N}\left[(\sum_{k} I_{ik}\mathbf{R}_k\cdot\mathbf{x}_i-y_i)I_{im}\mathbf{x}_i\right]\cdot\mathbf{w}_m=0,
	\end{equation}	
	which involves quadratic term of $\mathbf{w}$. Fortunately, the left side of \eqref{D.12} is actually $\frac{1}{z_m}\frac{\partial L}{\partial\mathbf{w}_m}|_2\cdot\mathbf{w}_m$. Notice that $z_m\neq0$ at local minima, otherwise $\frac{\partial L}{\partial\mathbf{w}_m}=0$ by \eqref{D.10} and \eqref{D.11} and $\mathbf{w}_m$ would have been treated like differentiable local minima. Combining \eqref{D.1} and \eqref{D.9} results in $\frac{\partial L}{\partial\mathbf{w}_m}\mid_2 \cdot\mathbf{w}_m=0$, thus we conclude that \eqref{D.12} is satisfied automatically and impose no additional constraint at all. \\
	
	\eqref{D.8} implies
	\begin{equation}\label{D.13}
	\sum_{i=1}^{N}\left[(\sum_{k}I_{ik}\mathbf{R}_k\cdot\mathbf{x}_i-y_i)I_{ij}\mathbf{x}_{i}\right]=0 \ \ \ \ (j \in [K]; j\neq m)
	\end{equation}
	
	\eqref{D.1} can be expressed by 
	\begin{equation}\label{D.14}
	\left(\frac{\partial L}{\partial\mathbf{w}_m}|_1\cdot\mathbf{x}_n\right)\mathbf{x}_n=\frac{\partial L}{\partial\mathbf{w}_m}|_1\cdot\left|\mathbf{x}_n\right|^2
	\end{equation}
	and the inequalities
	\begin{equation}\label{D.15}
	\frac{\partial L}{\partial\mathbf{w}_m}|_1\cdot\mathbf{x}_n<0,\ \frac{\partial L}{\partial\mathbf{w}_m}|_2\cdot\mathbf{x}_n>0.
	\end{equation}
	\eqref{D.14} indicates $\frac{\partial L}{\partial\mathbf{w}_m}|_1$ is parallel to $\mathbf{x}_n$, so is $\frac{\partial L}{\partial\mathbf{w}_m}|_2$ by \eqref{D.11}. \eqref{D.15} ensures $\mathbf{w}_m$ is a local minima (and cannot be local maximum or saddle point). Take $z_m\neq0$ into account, \eqref{D.14} can be written as the following form
	\begin{equation}\label{D.16}
	\sum_{i\neq n}{\left[(\sum_{k}{I_{ik}\mathbf{R}_k\cdot\mathbf{x}_i-y_i)}I_{im}(\mathbf{x}_i\cdot\mathbf{x}_n\mathbf{x}_n-\left|\mathbf{x}_n\right|^2\mathbf{x}_i)\right]=0\ }
	\end{equation}
	\eqref{D.9} can be transformed into 
	\begin{equation}\label{D.17}
	\mathbf{x}_n\cdot\mathbf{R}_m=0
	\end{equation}
	
	The inequalities in \eqref{D.15} are expanded as
	\begin{equation*}
	\sum_{i\neq n}{\left[(\sum_{k}{I_{ik}\mathbf{R}_k\cdot\mathbf{x}_i-y_i)}I_{im}\mathbf{x}_i\cdot\mathbf{x}_n\right]}z_m<0,
	\end{equation*}
	\begin{equation*}
	\begin{split}
	&\sum_{i\neq n}{\left[(\sum_{k}{I_{ik}\mathbf{R}_k\cdot\mathbf{x}_i-y_i)I_{im}\mathbf{x}_i\cdot\mathbf{x}_n}\right]z_m} \\
	+&\left[(\sum_{k}{I_{nk}\mathbf{R}_k\cdot\mathbf{x}_n-y_n)\left|\mathbf{x}_n\right|^2}\right]z_m>0,
	\end{split}
	\end{equation*}
	which has already appeared in \eqref{D.4} and \eqref{D.5}.
	
	Finally, \eqref{D.13}, \eqref{D.16} and \eqref{D.17} together form a linear system
	\begin{equation}\label{D.18}
	D\mathbf{R}=\mathbf{d}
	\end{equation}
	as defined in the statement of this theorem.
	The linear system \eqref{D.18} has solutions if and only if
	\begin{equation}\label{D.19}
	DD^+\mathbf{d}=\mathbf{d}
	\end{equation}
	
	If solvable, its general solution is 
	\begin{equation}\label{D.20}
	\mathbf{R}^\ast=D^+\mathbf{d}+\left(I-D^+D\right)\mathbf{c}\ \ \ (\mathbf{c}{\in\mathbb{R}}^{Kd}\ \textup{is arbitrary})
	\end{equation}
	If $ D$ is of full rank, then $\mathbf{R}^\ast=D^+\mathbf{d}$ is unique.\\
	
	We need to test whether the solution in \eqref{D.20} satisfies the constraints in \eqref{D.4} and \eqref{D.5}. If $\mathbf{R}^\ast$ is a single point, substituting $\mathbf{R}^\ast=D^+\mathbf{d}$ into \eqref{D.4} and \eqref{D.5}, then test with either $z_m>0$ or $z_m<0$. The magnitude of $z_m$ does not matter for \eqref{D.4} and \eqref{D.5}. Only if \eqref{D.19} holds, and the inequalities hold for  $z_m>0$ or $z_m<0$, there exist non-differentiable local minima. If $\mathbf{R}^\ast$ is a linear subspace of $\mathbb{R}^{Kd}$, substituting \eqref{D.20} into \eqref{D.4} and \eqref{D.5}, each inequality will define a half-space in $\mathbb{R}^{Kd}$. For example, \eqref{D.4} is transformed into 
	\begin{equation}\label{D.21}
	\begin{split}
	&\sum_{i\neq n}\left[(\sum_{k}{I_{ik}I_{im}\mathbf{x}_i^T\mathbf{x}_n\mathbf{x}_i^T{(I-D^+D)}_k}\right]z_m\cdot\mathbf{c} \\
	-&\sum_{i\neq n}\left[I_{im}y_i\mathbf{x}_i^T\mathbf{x}_n\right]z_m \\
	+&\sum_{i\neq n}{\left[\left(\sum_{k}{I_{ik}I_{im}\mathbf{x}_i^T\mathbf{x}_n\mathbf{x}_i^T{(D^+\mathbf{d})}_k}\right)\right]z_m<0},
	\end{split}	
	\end{equation}
	where\ ${(I-D^+D)}_k$ is the rows of $(I-D^+D)$ corresponding to $\mathbf{R}_k$, and so on.
\end{proof} 

\subsection{ Conditions for Existence of Genuine Non-differentiable Local Minima}\label{section6.1}
Like the case of differentiable local minima, existence of genuine non-differentiable local minima can be identified by testing against \eqref{eq11} and \eqref{eq12} if $\mathbf{R}^\ast$ is unique, or finding intersection of half-spaces like \eqref{eq13} if $\mathbf{R}^\ast$ is a linear subspace. The differences with differentiable local minima lie in that there is no need to test $\mathbf{w}_m^\ast $ since it is constrained on the  cell boundary, and instead the two inequalities \eqref{D.4} and \eqref{D.5} should be satisfied. If $\mathbf{R}^\ast$ is a linear subspace, the solutions to \eqref{eq13}, \eqref{D.4} and \eqref{D.5} can be obtained simultaneously by finding the intersection of corresponding half-spaces.

\section{Missing Proofs}\label{section7}

\setcounter{lemma}{1}
\begin{lemma}
	L$\left(\mathbf{R}_1,\mathbf{R}_2,\cdots\mathbf{R}_K\right)$ is convex inside cells if l is convex.
\end{lemma}

\begin{proof} [\indent \bfseries Proof]
	We will prove the convexity of L$\left(\mathbf{R}_1,\mathbf{R}_2,\cdots\mathbf{R}_K\right)$ by proving the positive definiteness of its Hessian. Notice that $\left \{ I_{ij},\ i\in [N], j\in [K] \right \} $ are constant inside cells. The derivative $\frac{\partial L}{\partial \mathbf{R}_{m}}=\frac{1}{N} \sum_{i=1}^{N} l^{\prime}\left(\sum_{j=1}^{K} I_{i j} \mathbf{R}_{j} \cdot \mathbf{x}_{i}, y_{i}\right) \cdot I_{i m} \mathbf{x}_{i}$ and the 2nd-order derivative is 
	\begin{equation}\label{f1}
	\frac{\partial^{2} L}{\partial \mathbf{R}_{m} \partial \mathbf{R}_{n}}=\frac{1}{N} \sum_{i=1}^{N} l^{\prime \prime} \cdot I_{i m} \mathbf{x}_{i} I_{i n} \mathbf{x}_{i}^{\mathrm{T}}
	\end{equation}
	Let $\mathbf{R}={(\mathbf{R}_1^T,\mathbf{R}_2^T,\cdots,\mathbf{R}_K^T)}^T$, then Hessian matrix $\frac{\partial^2L}{\partial\mathbf{R}^2}$ is a block matrix with block components $\frac{\partial^2L}{\partial\mathbf{R}_m\partial\mathbf{R}_n}$, $(m,n\in[K])$. Since $I_{im}$ is either 1 or 0, \eqref{f1} can be rewritten as
	\begin{equation}
	\begin{split}
	\frac{\partial^{2} L}{\partial \mathbf{R}_{m} \partial \mathbf{R}_{n}}=&\frac{1}{N} \sum_{i=1}^{N} l^{\prime \prime} \cdot\left(\begin{array}{c}
	{I_{i m} \cdot x_{i 1}} \\
	{I_{i m} \cdot x_{i 2}} \\
	{\vdots} \\
	{I_{i m} \cdot x_{i d}}
	\end{array}\right) \\
	&\left(I_{i n} \cdot x_{i 1} \quad I_{i n} \cdot x_{i 2} \quad \cdots \quad I_{i n} \cdot x_{i d}\right)
	\end{split}
	\end{equation}
	Defining $\mathbf{I}_{im}=\left(\begin{matrix}I_{im}&I_{im}\ldots.&I_{im}\\\end{matrix}\right)^T$ that repeats $I_{im}$ $ d $ times, and using the element-wise product $\odot$, there is $\left(\begin{array}{c}{I_{i m} \cdot x_{i 1}} \\ {I_{i m} \cdot x_{i 2}} \\ {\vdots} \\ {I_{i m} \cdot x_{i d}}\end{array}\right)=\mathbf{I}_{i m} \odot \mathbf{x}_{i}$. Then 
	\begin{equation*}
	\frac{\partial^{2} L}{\partial \mathbf{R}_{m} \partial \mathbf{R}_{n}}=\frac{1}{N} \sum_{i=1}^{N} l^{\prime \prime} \cdot \mathbf{I}_{i m} \odot \mathbf{x}_{i} \cdot\left(\mathbf{I}_{i n} \odot \mathbf{x}_{i}\right)^{\mathrm{T}}.
	\end{equation*}
	Let ${\widetilde{\mathbf{x}}}_i=\left(\begin{matrix}\begin{matrix}\mathbf{I}_{i1}\odot\mathbf{x}_i\\\mathbf{I}_{i2}\odot\mathbf{x}_i\\\end{matrix}\\\begin{matrix}\vdots\\\mathbf{I}_{iK}\odot\mathbf{x}_i\\\end{matrix}\\\end{matrix}\right)$, Hessian $\frac{\partial^2L}{\partial\mathbf{R}^2}$ can be transformed into 
	\begin{equation*}
	\frac{\partial^{2} L}{\partial \mathbf{R}^{2}}=\frac{1}{N} \sum_{i=1}^{N} l^{\prime \prime} \cdot \tilde{\mathbf{x}}_{i} \cdot \tilde{\mathbf{x}}_{i}^{\mathrm{T}}.
	\end{equation*}
	For arbitrary non-zero vector $\mathbf{u}{\in\mathbb{R}}^{Kd}$, the quadratic form $\mathbf{u}^{\mathrm{T}} \frac{\partial^{2} L}{\partial \mathbf{R}^{2}} \mathbf{u}=\frac{1}{N} \sum_{i=1}^{N} l^{\prime \prime} \cdot\left(\mathbf{u}^{\mathrm{T}} \tilde{\mathbf{x}}_{i}\right)^{2}$. By $l^{\prime\prime}>0$ due to convexity of $l$, the positive definiteness of Hessian and consequently the convexity of $L$ are obtained.  
\end{proof}

\section{Related Work}\label{section8}
\textbf{Loss landscape} Matrix completion and tensor decomposition, e.g., \cite{Matrixcompletion} are learning models involving the product of two unknown matrices, and it has been shown that all local minima are global for such models. Deep linear networks, which remove the non-linear activation function of each neuron in multi-layer perceptions, also have no spurious local minima according to \cite{Kawaguchi_nips16,HaihaoLu,Laurent18,GLOBALOPTIMALITY,Nouiehed,Zhang}. \cite{HardtMa} shows deep linear residual networks have no spurious local optima. \cite{Choromanska15} uses spin glass models in statistical physics to analyze the loss landscape which simplify the nonlinear nature of deep neural networks.

For one-hidden-layer over-parameterized networks with quadratic activation, \cite{Soltanolkotabi,QuadraticActivation} prove that all local minima are global. For one-hidden-layer ReLU networks, \cite{SoudryCarmon} gives the conditions under which loss at differentiable local minimum is zero (thus being global minimum). \cite{Multilinear} shows that ReLU networks with hinge loss can only have non-differentiable local minima and gives the conditions for their existence for linear separable data. \cite{SafranShamir16} shows that there is a high probability of initializing in a basin with small minimal loss for over-parameterized one-hidden-layer ReLU networks. \cite{SoudryHoffer} exhibits that, given standard Gaussian input data, the volume of differentiable regions containing sub-optimal differentiable local minima is exponentially vanishing in comparison with that containing global minima.

Absence of spurious valley for ultra-wide networks are explored in \cite{Venturi,Nguyen19,Li,Nguyenvalleys,Ding}. \cite{LossSurfaceBinary,AddingOneNeuron,Kawaguchi19} show that by adding a single-layer network or even a single special neuron in the shortcut connection, every local minimum becomes global. \cite{landdesign,Gao,Feizi} design new loss functions or special networks so that all local minima are global. \cite{ResNetsProvablyBetter,KawaguchiBengio} prove that depth with nonlinearity creates no bad local minima in a type of ResNets in the sense that the values of all local minima are no worse than that of global minima of corresponding shallow linear predictors. \cite{Spectrum,RandomMatrix} use random matrix theory to study the spectrum of Hessians of loss functions, which characterizes the landscape in the neighborhood of stationary points. \cite{Mei,Zhou_Feng} study the landscape of expected loss. \cite{Bianchini} shows the topological expressiveness advantage of deep networks over shallow ones in terms of bounds on the sum of Betti numbers. \cite{piecewiseconvex} shows that the loss surface of a feed-forward ReLU network regularized with weight decay is piecewise strongly convex on an important open set. 

\textbf{Saddle points} \cite{attackingthesaddle} argues that a main source of difficulty for local search based optimization methods comes from the proliferation of saddle points. \cite{GLOBALOPTIMALITY,CRITICALPOINTS,Kawaguchi_nips16} discuss saddle points for deep linear networks. \cite{howtoescape} designs a local search algorithm that can escape saddle points efficiently. Despite these works, concrete conditions for existence of saddle points were still missing for ReLU neural networks before this work.

\textbf{Empirical studies of landscape} Besides theoretical researches, there have been some experimental explorations on visualization of landscape \cite{Goodfellow,Poggio,Visualizing,largscalelandscape}, geometry of sub-level sets \cite{Topologygeometry} and mode connectivity \cite{EssentiallyNoBarriers,modeconnectivity}.

\textbf{Convergence of gradient based optimization} Understanding the landscape of loss functions focuses on the geomtry side of neural network optimization. Another line of research studies optimization of neural networks from a algebraic point of view by exploring the convergence of gradient based methods. These two lines of researches complement each other. Some recent works, e.g., \cite{SimonDu19,Allen-Zhu,QuanquanGu,OnehiddenlayerReLUviaGD} show that gradient descent converges for fully connected, convolutional and residual networks if they are sufficiently wide, the step-size is small enough and the initial weights have small magnitudes. Instead of small regions around global minima, in this work we consider the large scale structure of loss landscape for one-hidden-layer ReLU networks of any size. Convergence analysis for networks of any size and arbitrary initial weights still requires an understanding of global landscape.

\textbf{Comparisons with our work} The works most related to ours are \cite{SoudryCarmon,Multilinear,SafranShamir16,SoudryHoffer,SafranShamir18}, all of them dealing with local minima of one-hidden-layer ReLU networks. Comparing with our work, \cite{SoudryCarmon} considers only over-parameterized case, \cite{Multilinear} adopts hinge loss and linear separable data, while our theory is general and applies to one-hidden-layer ReLU networks of any size and any input. The experimental study of \cite{SafranShamir18} uses a student-teacher objective that is different than ours. \cite{SoudryHoffer} calculates the probability of having bad local minima. However, their concept of bad local minima is different from ours in the sense that they refer to local minima with nonzero losses, which are actually not genuine ones if locating outside their defining cells.   

\section{Conclusions}
We have studied from a theoretical persperctive the global loss landscape of one-hidden-layer ReLU networks, including the globalness of differentiable local minima, the conditions for existing differentiable and non-differentiable local minima and saddle points, and their locations and forms if they do exist.

In part 2 of this work \cite{globallosslandscape_part2}, we will describe how to implement efficient half-space intersection algorithm to judge the existence of genuine local minima when they are in the form of hyperplanes, and conduct experiments on both synthetic and real data to identify the existence of bad local minima and verify our theory. We will also investigate for Gaussian data how big the probability of existing bad local minima is at everywhere in whole weight space.


%

%
%
%
%
%

\ifCLASSOPTIONcaptionsoff
  \newpage
\fi



%

\bibliographystyle{plain}
\bibliography{Understanding_global_landscape_ReLU}

%
%
%

%

%
%
%




\end{document}